%% file: main.tex
\documentclass[11pt, english, submission, copyright, creativecommons]{eptcs}
\newcommand{\event}{CAPNS 2018}

\usepackage[utf8]{inputenc}
\usepackage{mathtools}
\usepackage{amssymb}
\usepackage{amsmath}
\usepackage{listings}
\usepackage{braket}
\usepackage[toc,page]{appendix}
\usepackage{proof}
\usepackage{color}
\usepackage{adjustbox}
\usepackage{physics}

\usepackage{graphicx}
\graphicspath{ {./images/} }

\newcommand{\s}{\enspace}

\DeclareMathOperator*{\argmax}{argmax}

\usepackage{amsthm}

\newtheorem{theorem}{Theorem}
\newtheorem{corollary}{Corollary}[theorem]

\newtheorem*{definition}{Definition}
\newtheorem{example}{Example}

\newcommand{\mat}[1]{\mathtt{#1}}
\renewcommand{\S}{S}
\newcommand{\E}{E}
\newcommand{\R}{R}

\usepackage{tikz,pgfplots}
\usepackage[strings]{underscore}

\usetikzlibrary{shapes.geometric}
\usetikzlibrary{decorations.markings}
\usetikzlibrary{
	arrows,calc,through,backgrounds,matrix,decorations.pathmorphing}

\pgfdeclarelayer{edgelayer}
\pgfdeclarelayer{nodelayer}
\pgfsetlayers{edgelayer,nodelayer,main}

\tikzstyle{none} = [inner sep=0pt]
\tikzstyle{black} = [circle, fill=black, scale=0.5]
\tikzstyle{white} = [draw, circle, scale=0.5]
\tikzstyle{antipode} = [draw, rectangle, fill=red, scale=0.5]
\tikzstyle{directed} = [->]

\def\bR{\begin{color}{red}}
\def\bB{\begin{color}{blue}}
\def\bM{\begin{color}{magenta}}
\def\bC{\begin{color}{cyan}}
\def\bW{\begin{color}{white}}
\def\bBl{\begin{color}{black}}
\def\bG{\begin{color}{green}}
\def\bY{\begin{color}{yellow}}
\def\e{\end{color}}

\title{Towards Compositional Distributional Discourse Analysis}
\author{
Bob Coecke, Giovanni de Felice, Dan Marsden \& Alexis Toumi
\email{\{firstname.lastname\}@cs.ox.ac.uk}
\institute{Department of Computer Science, University of Oxford}}
\date{December 2017}

\def\titlerunning{Compositional Distributional Discourse Analysis}
\def\authorrunning{B. Coecke, G. de Felice, D. Marsden \& A. Toumi}

\begin{document}
\maketitle
\input{0-introduction}
\input{1-preliminaries}
\input{2-sentences}
\input{3-questions-and-anaphoras}

\input{4-understanding-as-process}
\input{9-conclusion}
\newline
\bibliographystyle{eptcs}
\providecommand{\doi}[1]{\textsc{doi}:
\href{http://dx.doi.org/#1}{\nolinkurl{#1}}}
\bibliography{discocat}
\end{document}

%% file: 0-introduction.tex

\begin{abstract}
  Categorical compositional distributional semantics provide a method to derive the meaning of a sentence from the meaning of its individual words:
  the grammatical reduction of a sentence automatically induces a linear map for composing the word vectors obtained from distributional semantics.
  In this paper, we extend this passage from word-to-sentence to sentence-to-discourse composition.
  To achieve this we introduce a notion of basic anaphoric discourses as a mid-level representation between natural language discourse formalised in terms of basic \emph{discourse representation structures} (DRS); and knowledge base queries over the Semantic Web as described by \emph{basic graph patterns} in the Resource Description Framework (RDF).
  This provides a high-level specification for compositional algorithms for \emph{question answering} and \emph{anaphora resolution}, and allows us to give a picture of \emph{natural language understanding} as a process involving both statistical and logical resources.
\end{abstract}

\section*{Introduction}
In the last couple of decades, the traditional \emph{symbolic} approach to AI and cognitive science ---
which aims at characterising human intelligence in terms of abstract logical processes ---
has been challenged by so-called \emph{connectionist} AI: the study of the human brain as a complex network of basic processing units \cite{Smolensky87}.
When it comes to human language, the same divide manifests itself as the opposition between two principles, which in turn induce two distinct approaches to Natural Language Processing (\emph{NLP}).
On one hand Frege's principle of \emph{compositionality} asserts that the meaning of a complex expression is a function of its sub-expressions, and the way in which they are composed ---
\emph{distributionality} on the other hand can be summed up in Firth's maxim ``You shall know a word by the company it keeps".
Once implemented in terms of concrete algorithms we have
expert systems driven by formal logical rules on one end,
artificial neural networks and machine learning on the other.

Categorical Compositional Distributional (\emph{DisCoCat}) models, first introduced in \cite{DisCoCat08}, aim at getting the best of both worlds:
the string diagrams notation borrowed from category theory allows to manipulate the grammatical reductions as linear maps,
and compute graphically the semantics of a sentence as the composition of the vectors
which we obtain from the distributional semantics of its constituent words.

In this paper, we introduce \emph{basic anaphoric discourses} as mid-level representations between natural language discourse on one end --- formalised in terms of basic \emph{discourse representation structures} (DRS) \cite{Abramsky14}; and knowledge queries over the Semantic Web on the other --- given by \emph{basic graph patterns} in the Resource Description Framework (RDF) \cite{Stefanoni18}.
We construct discourses as formal diagrams of real-valued matrices and we then use these diagrams to give abstract reformulations of NLP problems: probabilistic \emph{anaphora resolution} and \emph{question answering}.

In the first two sections we introduce the notation used in the rest of the paper,
then give a brief summary of how DisCoCat models can be used to turn declarative sentences into probabilistic \texttt{Ask}-type queries in the \textsc{Sparql} protocol and RDF query language.
In the third section, we extend this analysis to \texttt{Select}-type queries and show how they can be translated as ``who" or ``whom" questions in the simple case,
as discourses making use of ambiguous pronouns such as ``they" or ``them" in more complex cases, i.e. when the \textsc{Sparql} query contains more than two output-variables.
In the last section we give a finer-grained analysis of anaphora resolution as a resource-sensitive process, copying word-vectors from memory and feeding them in the meaning computation of some anaphoric discourse.

We conclude our discussion with related work on deep neural networks and knowledge graph factorisation, as well as potential directions for modeling more involved linguistic phenomena and translating them in terms of knowledge base queries.

%% file: 1-preliminaries.tex
\section{Diagrams of Matrices for Knowledge Graphs}

First we fix an ordered set of \emph{scalars} $\S$ with addition $+$, multiplication $\times$ and units $0, 1 \in \S$ respectively,
then we write $\mat{M} : a \rightarrow b$ for $a \times b$ matrices with entries in $\S$ and $\mat{id}_{n} : a \rightarrow a$ for the $a \times a$ identity matrix.
For all $a, b, c \in \mathbb{N}$ and matrices $\mat{M}_1 : a \rightarrow b$, $\mat{M}_2 : b \rightarrow c$, we have their composition $\mat{M}_2 \mat{M}_1 : a \rightarrow c$ given by matrix multiplication.
We also have element-wise addition $\mat{M} + \mat{N}$ of matrices $\mat{M}, \mat{N} : a \rightarrow b$.
Hence given a choice of scalars (where $+$ and $\times$ respect the axioms of a semi-ring) we have the structure of a category $\mat{Mat}(\S)$, with the natural numbers as objects and matrices as arrows.
We will focus on the Boolean case $\S = \mathbb{B}$ with $+ = \lor$ and $\times = \land$ and on the non-negative reals $\S = \mathbb{R}^+$ encoding probabilities.
Note that a proper treatment of normalisation is beyond the scope of this paper.

We will manipulate matrices using the \emph{string diagrams} notation for monoidal categories --- for a reference guide to graphical languages see \cite{Selinger09}, for an introduction targeted at a general audience see \cite{Coecke09b} ---
we depict $\mat{id}_1 : 1 \rightarrow 1$ as the empty diagram and other identity matrices as labeled wires; we represent non-trivial matrices as boxes with labeled input and output wires and their composition as vertical concatenation.
Finally, horizontal juxtaposition of boxes depicts the Kronecker product
$\mat{M} \otimes \mat{N} : a \otimes c \rightarrow b \otimes d$ of matrices $\mat{M} : a \rightarrow b$, $\mat{N} : c \rightarrow d$.
We overload the notation and also write $a \otimes b$ for multiplication of natural numbers, i.e. the dimension of the tensor space.

We want to translate natural language into machine-readable format through linear algebra,
hence we reformulate the standard Semantic Web terminology in terms of vectors and matrices.
We assume that our vocabulary has been encoded in RDF: nouns come from an ordered set of \emph{entities} $\E = \set{Alice, boys, \dots}$, verbs come from an ordered set of \emph{relations} $\R = \set{loves, tell, \dots}$.
In Semantic Web languages, relations are always binary so we will focus on transitive verbs, for a translation of adjectives and intransitive verbs in RDF see \cite{Antonin14}. 
We encode every word as a binary code with one bit up, also called a \emph{one-hot vector} --- both entities $e_i \in \E$ and relations $r_j \in \R$ correspond to one-hot row vectors, of dimension $|\E|$ and $|\R|$ respectively.
We denote these one-hot vectors using Dirac's bra-ket notation:
$$
\ket{e_i} \enspace = \enspace \bigg( \lefteqn{\underbrace{
\phantom{0 \ \dots \ 0 \ 1 \ 0 \ \dots \ 0}}_{|\E|}}
\overbrace{0 \ \dots \ 0}^{i} \ 1 \ 0 \ \dots 0 \
\bigg) \enspace : \enspace 1 \rightarrow |\E|
\qquad \qquad
\ket{r_j} \enspace = \enspace \bigg( \lefteqn{\underbrace{
\phantom{0 \ \dots \ 0 \ 1 \ 0 \ \dots \ 0}}_{|\R|}}
\overbrace{0 \ \dots \ 0}^{j} \ 1 \ 0 \ \dots 0 \
\bigg) \enspace : \enspace 1 \rightarrow |\R|
\qquad
$$
The corresponding one-hot column vectors are denoted $\bra{e_i}: |\E| \rightarrow 1$ and $\bra{r_j} : |\R| \rightarrow 1$.
Row vectors are also called \emph{states}, in the diagrammatic language we depict them as triangles with no inputs;
similarly column vectors are called \emph{effects} which we depict as triangles with no outputs.

Composing a state with an effect of the same dimension yields a diagram with no inputs or ouputs: a scalar in $\S$.
Thus, in the case $\S = \mathbb{R}^+$ any effect $\mat{q} : n \rightarrow 1$ induces an ordering of the states $\mat{a}, \mat{b} : 1 \rightarrow n$, take $\mat{a} \leq \mat{b}$ if and only if $(\mat{q} \mat{a}) \leq (\mat{q} \mat{b})$.
When $\S = \mathbb{B}$, $\mat{q}$ yields a subset of the states: $\mat{a} : 1 \rightarrow n$ is in that subset if and only if $\mat{q} \mat{a} = 1$.
In both cases, the one-hot column vector $\bra{e} : |\E| \rightarrow 1$ corresponding to an entity $e \in \E$ acts as a test of identity with $e$: 
for all $e' \in \E$ we have $\bra{e}\ket{e'} = 1$ if and only if $e' = e$.

A \emph{knowledge graph} (also called an \emph{ontology}) is given by a set of \emph{RDF triples} $K \subseteq \E \times \R \times \E$
which contains all the true statements that we care about, e.g. $(Alice, loves, Bob), (Bob, hates, Charles) \in K$.
We can turn any knowledge graph into an effect by summing over the column vectors of its triples:
$$
\bra{K} \enspace = \enspace
\sum_{(s, v, o) \in K} \bra{s} \otimes \bra{v} \otimes \bra{o}
\enspace : \enspace | \E | \otimes | \R | \otimes | \E | \rightarrow 1
$$
Then for any triple $t \in \E \times \R \times \E$, if we compose the state
$\ket{t} : 1 \rightarrow | \E | \otimes | \R | \otimes | \E |$ with the effect $\bra{K}$, the scalar $\bra{K}\ket{t} \in \S$ we obtain is equal to $1$ if and only if $t \in K$--- we consider this scalar as the outcome of a basic knowledge graph query.
In order to study more complex queries, we need a way to wire knowledge graphs together: we do this using the \emph{generalised Kronecker delta} together with the following standard results.

\begin{definition}
For $a, b, n \in \mathbb{N}$, the $(a, b)$ \emph{Kronecker delta} over $n$
--- which we depict as an $a$-input, $b$-output \emph{spider} --- is the matrix:
\begin{equation}
\delta_n^{a, b} \enspace = \enspace
\sum_{i < n} \enspace \ket{e_i}^{(\otimes b)} \bra{e_i}^{(\otimes a)}
\enspace : \enspace n^{a} \rightarrow n^{b}
\end{equation}
\end{definition}
\begin{theorem}
\label{spiderFusion}
For all $n \in \mathbb{N}$, the set of matrices
$\big\{ \delta_n^{a, b} \big\}_{a, b \in \mathbb{N}}$
obeys the \emph{spider fusion} equations:
\input{images/1-0-spider-fusion}
\end{theorem}
\begin{proof}
We have: $
\big( \mat{id}_{n}^{(\otimes b)} \otimes \delta_{n}^{(k + c), d} \big)
\enspace
\big( \delta_{n}^{a, (b + k)} \otimes \mat{id}_{n}^{(\otimes c)} \big)
\ \ = \ \ \delta_{n}^{(a + c), (b + d)}
$, see \cite{Coecke09c}.
Note how the string diagrams make the topological nature of the equation apparent while hiding the bureaucracy of indices.
\end{proof}
The $(2, 0)$ and $(0, 2)$ Kronecker deltas
--- which we depict as \emph{cups} and \emph{caps} respectively ---
make $\mat{Mat}(\S)$ a \emph{compact closed} category. We recall a couple of useful properties of cups and caps.

\begin{corollary}
For all $n \in \mathbb{N}$,
the effect $\mat{cup}_{n} = \delta_{n}^{2, 0} : n \otimes n \rightarrow 1$
and the state $\mat{cap}_{n} = \delta_{n}^{0, 2} : 1 \rightarrow n \otimes n$
obey the \emph{snake equations}:
\input{images/1-1-snake-equations}
\end{corollary}
\begin{proof}
We have:$ \enspace
(\mat{id}_{n} \otimes \mat{cup}_{n}) \enspace (\mat{cap}_{n} \otimes \mat{id}_{n})
\enspace = \enspace \mat{id}_{n} \enspace = \enspace 
(\mat{cup}_{n} \otimes \mat{id}_{n}) \enspace (\mat{id}_{n} \otimes \mat{cap}_{n})
\enspace $
which follows from Theorem \ref{spiderFusion}
and $\enspace \delta_{n}^{1, 1} \enspace
= \enspace \sum_{i < n} \ket{e_i}\bra{e_i} \enspace
= \enspace \mat{id}_{n}$, see \cite{Coecke17}.
\end{proof}
\begin{corollary}
We can define the transpose of a matrix $\mat{M} : m \rightarrow n$ as:
\input{images/1-2-transpose}
\begin{proof}
We have
$\mat{M}^T
= (\mat{id}_{n} \otimes \mat{cup}_{m}) \enspace
(\mat{id}_{n} \otimes \mat{M} \otimes \mat{id}_{m}) \enspace
(\mat{cap}_{n} \otimes \mat{id}_{m})
\enspace : \enspace n \rightarrow m$, see \cite{Coecke17}.
\end{proof}
\end{corollary}

%% file: images/1-0-spider-fusion.tex
$$
\begin{tikzpicture} [scale=0.666, baseline={([yshift=-.5ex]current bounding box.center)},vertex/.style={anchor=south, circle,fill=black!25,minimum size=18pt,inner sep=2pt}]

\node [draw, circle, scale=0.5] (0) at (-1, 0.5) {};

\node [draw, circle, scale=0.5] (1) at (1, -0.5) {};

\draw (0) to[out=0, in=-90] (0, 1.5);
\draw (0) to[out=-100, in=90] (-1, -1.5);
\draw (0) to[out=180, in=-90] (-2.5, 1.5);
\draw (0) to[out=180, in=90] (-2.5, -1.5);

\draw [decorate,decoration={brace,amplitude=5},xshift=0.5]
(-2.6, 1.6) -- (0.1, 1.6);
\node at (-1.25, 2.25) {$a$};

\draw [decorate,decoration={brace,amplitude=5},xshift=0.5]
(-0.9,-1.6) -- (-2.6,-1.6);
\node at (-1.75, -2.25) {$b$};

\draw (0) to[out=0, in=90] (1);
\draw (0) to[out=-90, in=180] (1);

\draw (1) to[out=0, in=-90] (2.5, 1.5);
\draw (1) to[out=0, in=90] (2.5, -1.5);
\draw (1) to[out=80, in=-90] (1, 1.5);
\draw (1) to[out=180, in=90] (0, -1.5);

\draw [decorate,decoration={brace,amplitude=5},xshift=0.5]
(0.9, 1.6) -- (2.6, 1.6);
\node at (1.75, 2.25) {$c$};

\draw [decorate,decoration={brace,amplitude=5},xshift=0.5]
(2.6,-1.6) -- (-0.1,-1.6);
\node at (1.25, -2.25) {$d$};

\node at (0, 0) {$\dots$};
\node at (-1.2,1.4) {$\dots$};
\node at (1.2,-1.4) {$\dots$};
\node at (1.8, 1.4) {$\dots$};
\node at (-1.75,-1.4) {$\dots$};
\end{tikzpicture}
=
\quad
\begin{tikzpicture} [scale=0.5, baseline={([yshift=-.5ex]current bounding box.center)},vertex/.style={anchor=south, circle,fill=black!25,minimum size=18pt,inner sep=2pt}]

\node [draw, circle, scale=0.5] (0) at (0,0) {};

\draw (0) to[out=0, in=-90] (1.5, 1.5);
\draw (0) to[out=0, in=90] (1.5, -1.5);
\draw (0) to[out=180, in=-90] (-1, 1.5);
\draw (0) to[out=180, in=90] (-1, -1.5);
\draw (0) to[out=180, in=-90] (-1.5, 1.5);
\draw (0) to[out=180, in=90] (-1.5, -1.5);

\node at (0.5, 1.4) {$\dots$};
\node at (0.5, -1.4) {$\dots$};

\draw [decorate,decoration={brace,amplitude=5},xshift=0.5]
(-1.6,1.6) -- (1.6,1.6);
\node at (0, 2.25) {$a + c$};

\draw [decorate,decoration={brace,amplitude=5},xshift=0.5]
(1.6,-1.6) -- (-1.6,-1.6);
\node at (0, -2.25) {$b + d$};
\end{tikzpicture}
$$

%% file: images/1-1-snake-equations.tex
$$
\begin{tikzpicture} [scale=0.666, baseline={([yshift=-.5ex]current bounding box.center)},vertex/.style={anchor=south, circle,fill=black!25,minimum size=18pt,inner sep=2pt}]
\draw (-1, 1)--(-1, 0);

\draw (-1, 0) to [out=-90, in=180] (0, -1);
\node at (0, -0.75) {$n$};
\draw (0, -1) to [out=0, in=-90] (1, 0);

\draw (1, 0) to [out=90, in=180] (2, 1);
\node at (2, 0.75) {$n$};
\draw (2, 1) to [out=0, in=90] (3, 0);

\draw (3, 0)--(3, -1);
\end{tikzpicture}
\quad
=
\quad
\begin{tikzpicture} [scale=0.666, baseline={([yshift=-.5ex]current bounding box.center)},vertex/.style={anchor=south, circle,fill=black!25,minimum size=18pt,inner sep=2pt}]
\draw (0,-1) to (0,1);
\node at (0.5, 0) {$n$};
\end{tikzpicture}
\quad = \quad
\begin{tikzpicture} [scale=0.666, baseline={([yshift=-.5ex]current bounding box.center)},vertex/.style={anchor=south, circle,fill=black!25,minimum size=18pt,inner sep=2pt}]
\draw (-1, -1)--(-1, 0);

\draw (-1, 0) to [out=90, in=180] (0, 1);
\node at (0, 0.75) {$n$};
\draw (0, 1) to [out=0, in=90] (1, 0);

\draw (1, 0) to [out=-90, in=180] (2, -1);
\node at (2, -0.75) {$n$};
\draw (2, -1) to [out=0, in=-90] (3, 0);

\draw (3, 0)--(3, 1);
\end{tikzpicture}
$$

%% file: images/1-2-transpose.tex
$$
\begin{tikzpicture} [scale=0.666, baseline={([yshift=-.5ex]current bounding box.center)},vertex/.style={anchor=south, circle,fill=black!25,minimum size=18pt,inner sep=2pt}]
\node at (2, 2.5) {$m$};
\draw (1.5, 2)--(1.5, 3);
\draw (0, 0) rectangle node {$\mat{M}^T$} (3, 2);
\draw (1.5, 0)--(1.5, -1);
\node at (2, -0.5) {$n$};
\end{tikzpicture}
\qquad = \qquad
\begin{tikzpicture} [scale=0.666, baseline={([yshift=-.5ex]current bounding box.center)},vertex/.style={anchor=south, circle,fill=black!25,minimum size=18pt,inner sep=2pt}]

\draw (-1.5, 2) to [out=90, in=180] (0, 3);
\node at (0, 2.5) {$n$};
\draw (0, 3) to [out=0, in=90] (1.5, 2);

\node at (5.0, 2.5) {$m$};
\draw (4.5, 3)--(4.5, 0);

\draw (0, 0) rectangle node {$\mat{M}$} (3, 2);

\node at (-1.0, -0.5) {$n$};
\draw (-1.5, 2)--(-1.5, -1);

\draw (1.5, 0) to [out=-90, in=180] (3, -1);
\node at (3, -0.5) {$m$};
\draw (3, -1) to [out=0, in=-90] (4.5, 0);
\end{tikzpicture}
$$

%% file: 2-sentences.tex
\section{Diagrams for Distributional Compositional Semantics}
Lambek introduced~\emph{pregroup grammars} as a means of encoding the grammatical structure of natural language.
In~\cite{DisCoCat08, DisCoCat11}, Clark et al.~observe that both pregroups and vector spaces carry the same mathematical structure:
both are compact closed categories.
This allows us to assign vectors to words and automatically combine them together using a linear map induced by the grammatical reduction.
In this section, we give a graphical presentation of the concrete implementation described in \cite{Grefenstette11} and reformulate their construction in terms of knowledge graphs.

The first step is to build a \emph{noun space} $N$ which we will assign to the noun type of our pregroup grammar.
This can be done by collecting co-occurrence data over a large natural language corpus:
the dimension $n$ of the noun space can be seen as a hyper-parameter of the model,
where in the simple case each dimension corresponds to a ``context word".
However, our model is agnostic of the exact method used to construct $N$:
we will only assume that every entity $e$ is assigned a vector $v_e \in N$,
and store them as one matrix $\mat{E} = \sum_{e \in E} v_e \bra{e}$.
This encoding matrix
$\mat{E} : | \E | \rightarrow n$
gives us a measure of similarity between the entities, indeed whereas we have $\bra{cat}\ket{dog} = 0$ now we expect the inner product
$\bra{cat} \mat{E}^T \mat{E} \ket{dog} \in \mathbb{R}^+$ to be strictly greater than zero:
their contexts share at least the words ``pet", ``food", etc.

Then for every verb $v \in R$, we compute its meaning by summing over the outer products $\mat{E} \ket{s} \otimes \mat{E} \ket{o} : 1 \rightarrow n \otimes n$ of all the encoded entities that it relates, i.e. all $s, o \in E$ such that $(s, v, o) \in K$.
Given an encoding matrix $\mat{E}$ and a knowledge graph $K$ we store this construction as a matrix:
\input{images/2-0-r-matrix}
encoding every relation $v$ in the \emph{verb space} $N \otimes N$.
Finally, we can compute the semantics of ``subject verb object" sentences using the following recipe:
let $\mat{T} = \mat{E} \otimes \mat{R} \otimes \mat{E}$ be the encoding matrix for triples,
\begin{enumerate}
\item tensor the encoded states for $t = ( s, v, o ) \in \E \times \R \times \E$ together
$$
\mat{T} \ket{t} = \mat{E} \ket{s} \otimes \mat{R} \ket{v} \otimes \mat{E} \ket{o} : 1 \rightarrow n \otimes (n \otimes n) \otimes n
$$
\item translate the grammar into a linear map in $\mat{Mat}_\mathbb{R}^+$, here the effect
$$
\mat{G} = \mat{cup}_n \otimes \mat{cup}_n : (n \otimes n) \otimes (n \otimes n) \rightarrow 1
$$
\item compose word meanings with grammar to get a scalar
$\mat{G} \mat{T} \ket{t} \in \mathbb{R}^+$.
\end{enumerate}
Computing these real-valued scalars allows to generalise to unseen sentences and outperform the standard bag-of-words approach on word disambuigation tasks, where the experimental data can be seen as a small knowledge graph \cite{Grefenstette11}.
We can give the semantics of more complex sentences --- hence more complex knowledge graph patterns --- by modeling relative pronouns as three-output spiders over $n$. For any $\mat{x}, \mat{y} : 1 \rightarrow n$, if we encode the word ``that" as the Kronecker delta $\delta_n^{0, 3} : n^3 \rightarrow 1$ and apply the appropriate grammatical reduction we get:
$$
\mat{G'} \enspace \big( \enspace \mat{x} \otimes that \otimes \mat{y} \enspace \big)
\enspace = \enspace
\delta_n^{2, 1} \enspace \big( \enspace \mat{x} \otimes \mat{y} \enspace \big)
\enspace : \enspace 1 \rightarrow n
$$
where $\mat{G'} = \mat{cup}_n \otimes \mat{id}_n \otimes \mat{cup}_n : n^5 \rightarrow 1$.
This the coordinate-wise multiplication of the vectors $\mat{x}$ and $\mat{y}$, which can be seen as the intersection of entities in co-occurrence space, see \cite{Frobenius1} and \cite{Frobenius2}.
\begin{example}
If we encode some philosophy into our knowledge graph $K$, we want that:
\input{images/2-1-0-intersection}
i.e. all men are mortal.
\end{example}
\begin{example}
We compute the semantics of ``Alice loves boys that tell jokes." as the following scalar:
\input{images/2-1-alice-loves-boys}
\end{example}

%% file: images/2-0-r-matrix.tex
$$
\begin{tikzpicture} [scale=0.666, baseline={([yshift=-.5ex]current bounding box.center)},vertex/.style={anchor=south, circle,fill=black!25,minimum size=18pt,inner sep=2pt}]

\node at (2, 2.5) {$| \R |$};
\draw (1.5, 2)--(1.5, 3);
\draw (0, 0) rectangle node {$\mat{R}$} (3, 2);
\draw (1, 0)--(1, -0.5);
\node at (0.5, -0.25) {$n$};
\draw (2, 0)--(2, -0.5);
\node at (2.5, -0.25) {$n$};
\end{tikzpicture}
\quad = \quad
\begin{tikzpicture} [scale=0.666, baseline={([yshift=-.5ex]current bounding box.center)},vertex/.style={anchor=south, circle,fill=black!25,minimum size=18pt,inner sep=2pt}]

\draw (0, 0) rectangle node {$\mat{E}$} (3, 2);
\draw (1.5, 0)--(1.5, -0.5);
\node at (1, -0.25) {$n$};

\node at (5.5, 2.5) {$| \R |$};
\draw (5, 2)--(5, 3);
\draw (3.5, 2)--(5, 0)--(6.5, 2)--(3.5, 2);
\node at (5, 1) {$\bra{K}$};

\draw (7, 0) rectangle node {$\mat{E}$} (10, 2);
\draw (8.5, 0)--(8.5, -0.5);
\node at (9, -0.25) {$n$};

\draw (1.5, 2) to [out=90, in=180] (2.75, 3);
\node at (2.75, 2.5) {$| \E |$};
\draw (2.75, 3) to [out=0, in=90] (4, 2);

\draw (6, 2) to [out=90, in=180] (7.25, 3);
\node at (7.25, 2.5) {$| \E |$};
\draw (7.25, 3) to [out=0, in=90] (8.5, 2);
\end{tikzpicture}
$$

%% file: images/2-1-0-intersection.tex
$$\begin{tikzpicture}[scale=0.666, baseline={([yshift=-.5ex]current bounding box.center)},vertex/.style={anchor=south, circle,fill=black!25,minimum size=18pt,inner sep=2pt}]

\node at (8.5, 3.5) {$Men$};
\draw (7.0, 3)--(10.0, 3)--(8.5, 4.5)--(7.0, 3);

\node at (9, 2.5) {$| \E |$};
\draw (8.5, 2)--(8.5, 3);
\draw (7, 0) rectangle node {$\mat{E}$} (10, 2);

\node at (12.5, 3.5) {$that$};
\node at (13.25, 2.375) {$\delta_n^{0, 3}$};

\node [draw, circle, scale=0.5] (1) at (12.5, 2) {};
\draw (1) to [out=180, in=90] (11.75, 0);
\draw (1) to [out=0, in=90] (13.25, 0);

\node at (16.5, 3.5) {$are$};
\draw (15.0, 3)--(18.0, 3)--(16.5, 4.5)--(15.0, 3);

\node at (17, 2.5) {$| \R |$};
\draw (16.5, 2)--(16.5, 3);
\draw (15, 0) rectangle node {$\mat{R}$} (18, 2);

\node at (20, 3.5) {$mortal$};
\draw (18.5, 3)--(21.5, 3)--(20.0, 4.5)--(18.5, 3);

\node at (20.5, 2.5) {$| \E |$};
\draw (20, 2)--(20, 3);
\draw (18.5, 0) rectangle node {$\mat{E}$} (21.5, 2);

\draw (8.5, 0) to [out=-90, in=180] (10, -1);
\node at (10, -0.75) {$n$};
\draw (10, -1) to [out=0, in=-90] (11.75, 0);

\draw (13.25, 0) to [out=-90, in=180] (14.5, -1);
\node at (14.5, -0.75) {$n$};
\draw (14.5, -1) to [out=0, in=-90] (16, 0);

\draw (17, 0) to [out=-90, in=180] (18.5, -1);
\node at (18.5, -0.75) {$n$};
\draw (18.5, -1) to [out=0, in=-90] (20, 0);

\draw (1)--(12.5, -1);
\node at (13, -0.75) {$n$};

\end{tikzpicture}
\qquad = \qquad
\begin{tikzpicture}[scale=0.666, baseline={([yshift=-.5ex]current bounding box.center)},vertex/.style={anchor=south, circle,fill=black!25,minimum size=18pt,inner sep=2pt}]
\node at (8.5, 3.5) {$Men$};
\draw (7.0, 3)--(10.0, 3)--(8.5, 4.5)--(7.0, 3);

\node at (9, 2.5) {$| \E |$};
\draw (8.5, 2)--(8.5, 3);
\draw (7, 0) rectangle node {$\mat{E}$} (10, 2);

\draw (8.5, 0)--(8.5, -1);
\node at (9, -0.75) {$n$};
\end{tikzpicture}
$$

%% file: images/2-1-alice-loves-boys.tex
$$\begin{tikzpicture}[scale=0.666]
\node at (1.5, 3.5) {$Alice$};
\draw (0.0, 3)--(3.0, 3)--(1.5, 4.5)--(0.0, 3);

\node at (2, 2.5) {$| \E |$};
\draw (1.5, 2)--(1.5, 3);
\draw (0, 0) rectangle node {$\mat{E}$} (3, 2);

\node at (5.0, 3.5) {$loves$};
\draw (3.5, 3)--(6.5, 3)--(5.0, 4.5)--(3.5, 3);

\node at (5.5, 2.5) {$| \R |$};
\draw (5, 2)--(5, 3);
\draw (3.5, 0) rectangle node {$\mat{R}$} (6.5, 2);

\node at (8.5, 3.5) {$boys$};
\draw (7.0, 3)--(10.0, 3)--(8.5, 4.5)--(7.0, 3);

\node at (9, 2.5) {$| \E |$};
\draw (8.5, 2)--(8.5, 3);
\draw (7, 0) rectangle node {$\mat{E}$} (10, 2);

\node at (12.5, 3.5) {$that$};
\node at (13.25, 2.375) {$\delta_n^{0, 3}$};

\node [draw, circle, scale=0.5] (1) at (12.5, 2) {};
\draw (1)--(12.5, 0);
\draw (1) to [out=180, in=90] (11.75, 0);
\draw (1) to [out=0, in=90] (13.25, 0);

\node at (16.5, 3.5) {$tell$};
\draw (15.0, 3)--(18.0, 3)--(16.5, 4.5)--(15.0, 3);

\node at (17, 2.5) {$| \R |$};
\draw (16.5, 2)--(16.5, 3);
\draw (15, 0) rectangle node {$\mat{R}$} (18, 2);

\node at (20, 3.5) {$jokes$};
\draw (18.5, 3)--(21.5, 3)--(20.0, 4.5)--(18.5, 3);

\node at (20.5, 2.5) {$| \E |$};
\draw (20, 2)--(20, 3);
\draw (18.5, 0) rectangle node {$\mat{E}$} (21.5, 2);

\node at (22, 3.5) {$.$};

\draw (1.5, 0) to [out=-90, in=180] (3, -1);
\node at (3, -0.75) {$n$};
\draw (3, -1) to [out=0, in=-90] (4.5, 0);

\draw (8.5, 0) to [out=-90, in=180] (10, -1);
\node at (10, -0.75) {$n$};
\draw (10, -1) to [out=0, in=-90] (11.75, 0);

\draw (13.25, 0) to [out=-90, in=180] (14.5, -1);
\node at (14.5, -0.75) {$n$};
\draw (14.5, -1) to [out=0, in=-90] (16, 0);

\draw (17, 0) to [out=-90, in=180] (18.5, -1);
\node at (18.5, -0.75) {$n$};
\draw (18.5, -1) to [out=0, in=-90] (20, 0);

\draw (5.5, 0) to [out=-90, in=180] (8.875, -2);
\node at (8.875, -1.75) {$n$};
\draw (8.875, -2) to [out=0, in=-90] (12.5, 0);

\end{tikzpicture}$$

%% file: 3-questions-and-anaphoras.tex
\section{From Questions and Discourse to Knowledge Graph Queries}

From an NLP perspective, computing the diagram for the sentence above can equivalently be seen as answering the question ``Do the boys that Alice loves tell jokes?".
We will exploit this question-answer duality to show how natural language can be translated into machine-readable format in a structure-preserving way.
Given a concrete implementation of a knowledge graph, also called a \emph{triplestore}, the semantics of yes-no questions can be interpreted as the \texttt{true} - \texttt{false} output of \textsc{Ask}-type queries in the \textsc{Sparql} protocol and RDF query language.
As a first step towards more complex query patterns, we extend this language-to-database translation to object and subject questions as \textsc{Select}-type queries.

In \cite{Lambek08}, Lambek gives a thorough analysis of the grammar of questions in terms of pregroups.
Here we will leave the subtleties of the English language aside and only assume designated grammatical types for object and subject questions;
in our semantic category $\mat{Mat}(\S)$, both correspond to effects $\mat{q} : | \E | \rightarrow 1$.
Indeed, given a potential answer $a \in E$ the inner product $\mat{q} \ket{a} \in \S$ gives us a measure of how well $a$ answers the query $\mat{q}$;
comparing these scalars allow us to translate a natural language question into a ranking of its possible answers.
Note that if we interpret $\mat{q}$ with $\S = \mathbb{B}$ we get crisp answers while taking $\S = \mathbb{R}^+$ and the encoding $\mat{T} = \mat{E} \otimes \mat{R} \otimes \mat{E}$ yields real-valued approximations.

In order to form a subject question, Lambek starts by constructing incomplete sentences with a hole in the subject position, e.g. ``- loves Bob". Object questions proceed dually, for example ``Alice loves -".
Question words like ``who" and ``whom" can then be seen as processes from these incomplete sentences to the question type.
In our setup, the semantics of both words are given by the encoding matrix $\mat{E} : |\E| \rightarrow n$, they take the state for the answer as input and feed it in the computation of the corresponding query.
The semantics for ``Who loves Bob?" and ``Who does Alice love?" are given by:
\input{images/3-0-who-questions}
Note that the semantics of ``does" is given by a cap, i.e. a maximally-entangled state of two noun sytems, this is discussed further in \cite{Clark13}.
Indeed, using the symmetry
$(\bra{x} \otimes \bra{y}) \enspace \mat{cap}_{| \E |} \enspace = \enspace (\bra{y} \otimes \bra{x}) \enspace \mat{cap}_{| \E |}$ for all $x, y \in E$
and the snake equations we can rewrite object questions and see the cap serve as an information-passing mechanism:
\input{images/3-1-alice-loves-whom}

We can extend this to two-variable \textsc{Select} queries such as

\input{images/3-2-who-loves-whom}

however, general \textsc{Select} queries can have an unbounded number of variables whereas natural language questions of three arguments or more will sound awkward to a human speaker --- picture a stranger asking you: ``Who did what where with whom?"...
Instead, we argue that natural language speakers can express complex query patterns through ambiguous \emph{discourse} together with \emph{anaphora} as a resource for interaction of meanings.
Concretely, we translate personal pronouns such as ``us" and ``them" in terms of \textsc{Sparql} query variables ranging over the set of entities $\E$,
then we say a sentence is \emph{anaphoric} when it contains a personal pronoun: its meaning depends on that of another expression in context.

We define an \emph{atomic sentence} as one of the form ``subject verb object",
where both subject and object may be either entities or personal pronouns.
In the anaphoric case we obtain a diagram with an open wire for the anaphoric expression, representing the input required from the context in order to compute the meaning of the sentence.
Once the ambiguity has been resolved, i.e. once we have assigned an entity to the pronoun, we get a diagram with no open wires as in the unambiguous case of Section 2. 
Finally a \emph{basic anaphoric discourse} is defined as a sequence of these atomic sentences, when it has $k$ pronouns in total its semantics is an effect of dimension
$|\E|^k$: a knowledge graph query of $k$ variables.

\input{30-basic-discourse}

%% file: images/3-0-who-questions.tex
\begin{equation*}
\begin{aligned}
\begin{tikzpicture} [scale=0.666, baseline={([yshift=-.5ex]current bounding box.center)},vertex/.style={anchor=south, circle,fill=black!25,minimum size=18pt,inner sep=2pt}]
\draw (0, 1.5)--(0, 0);
\node at (0.5, 0.75) {$| \E |$};

\node at (0, -1) {$\mat{q}_S$};
\draw (-2, 0)--(2, 0)--(0, -2)--(-2, 0);
\end{tikzpicture}
\quad =& \quad
\begin{tikzpicture} [scale=0.666, baseline={([yshift=-.5ex]current bounding box.center)},vertex/.style={anchor=south, circle,fill=black!25,minimum size=18pt,inner sep=2pt}]

\node at (0, 3.5) {$Who$};

\node at (2.25, 2.5) {$| \E |$};
\draw (1.5, 2)--(1.5, 4.5);
\draw (0, 0) rectangle node {$\mat{E}$} (3, 2);

\node at (5.0, 3.5) {$loves$};
\draw (3.5, 3)--(6.5, 3)--(5.0, 4.5)--(3.5, 3);

\node at (5.75, 2.5) {$| \R |$};
\draw (5, 2)--(5, 3);
\draw (3.5, 0) rectangle node {$\mat{R}$} (6.5, 2);

\node at (8.5, 3.5) {$Bob$};
\draw (7.0, 3)--(10.0, 3)--(8.5, 4.5)--(7.0, 3);

\node at (9.25, 2.5) {$| \E |$};
\draw (8.5, 2)--(8.5, 3);
\draw (7, 0) rectangle node {$\mat{E}$} (10, 2);

\node at (10.5, 3.5) {$?$};

\draw (1.5, 0) to [out=-90, in=180] (3, -1);
\node at (3, -0.75) {$n$};
\draw (3, -1) to [out=0, in=-90] (4.5, 0);

\draw (5.5, 0) to [out=-90, in=180] (7, -1);
\node at (7, -0.75) {$n$};
\draw (7, -1) to [out=0, in=-90] (8.5, 0);
\end{tikzpicture}
\\
\begin{tikzpicture} [scale=0.666, baseline={([yshift=-.5ex]current bounding box.center)},vertex/.style={anchor=south, circle,fill=black!25,minimum size=18pt,inner sep=2pt}]
\draw (0, 1.5)--(0, 0);
\node at (0.5, 0.75) {$| \E |$};

\node at (0, -1) {$\mat{q}_O$};
\draw (-2, 0)--(2, 0)--(0, -2)--(-2, 0);
\end{tikzpicture}
\quad =& \quad
\begin{tikzpicture} [scale=0.666, baseline={([yshift=-.5ex]current bounding box.center)},vertex/.style={anchor=south, circle,fill=black!25,minimum size=18pt,inner sep=2pt}]
\node at (0, 3.5) {$Who$};

\node at (2.25, 2.5) {$| \E |$};
\draw (1.5, 2)--(1.5, 4.5);
\draw (0, 0) rectangle node {$\mat{E}$} (3, 2);

\node at (5.0, 3.5) {$does$};

\draw (3.5, 0) to [out=90, in=180] (5, 2);
\node at (5, 1.5) {$n$};
\draw (5, 2) to [out=0, in=90] (6.5, 0);

\node at (8.5, 3.5) {$Alice$};
\draw (7.0, 3.0)--(10.0, 3.0)--(8.5, 4.5)--(7.0, 3);

\node at (9.25, 2.5) {$| \E |$};
\draw (8.5, 2)--(8.5, 3);
\draw (7, 0) rectangle node {$\mat{E}$} (10, 2);

\node at (12, 3.5) {$love$};
\draw (10.5, 3.0)--(13.5, 3)--(12, 4.5)--(10.5, 3);

\node at (12.75, 2.5) {$| \R |$};
\draw (12, 2)--(12, 3);
\draw (10.5, 0) rectangle node {$\mat{R}$} (13.5, 2);

\node at (14, 3.5) {$?$};

\draw (1.5, 0) to [out=-90, in=180] (2.5, -1);
\node at (2.5, -0.75) {$n$};
\draw (2.5, -1) to [out=0, in=-90] (3.5, 0);

\draw (8.5, 0) to [out=-90, in=180] (10, -1);
\node at (10, -0.75) {$n$};
\draw (10, -1) to [out=0, in=-90] (11.5, 0);

\draw (6.5, 0) to [out=-90, in=180] (9.5, -2);
\node at (9.5, -1.75) {$n$};
\draw (9.5, -2) to [out=0, in=-90] (12.5, 0);
\end{tikzpicture}
\end{aligned}
\end{equation*}

%% file: images/3-1-alice-loves-whom.tex
\begin{equation*}
\begin{aligned}
\begin{tikzpicture} [scale=0.666, baseline={([yshift=-.5ex]current bounding box.center)},vertex/.style={anchor=south, circle,fill=black!25,minimum size=18pt,inner sep=2pt}]
\draw (0, 1.5)--(0, 0);
\node at (0.5, 0.75) {$| \E |$};

\node at (0, -1) {$\mat{q}_O$};
\draw (-2, 0)--(2, 0)--(0, -2)--(-2, 0);
\end{tikzpicture}
\quad = \quad
\begin{tikzpicture} [scale=0.666, baseline={([yshift=-.5ex]current bounding box.center)},vertex/.style={anchor=south, circle,fill=black!25,minimum size=18pt,inner sep=2pt}]

\node at (1.5, 3.5) {$Alice$};
\draw (0.0, 3)--(3.0, 3)--(1.5, 4.5)--(0.0, 3);

\node at (2, 2.5) {$| \E |$};
\draw (1.5, 2)--(1.5, 3);
\draw (0, 0) rectangle node {$\mat{E}$} (3, 2);

\node at (5.0, 3.5) {$loves$};
\draw (3.5, 3)--(6.5, 3)--(5.0, 4.5)--(3.5, 3);

\node at (5.75, 2.5) {$| \R |$};
\draw (5, 2)--(5, 3);
\draw (3.5, 0) rectangle node {$\mat{R}$} (6.5, 2);

\node at (10, 3.5) {$whom?$};

\node at (9.25, 2.5) {$| \E |$};
\draw (8.5, 2)--(8.5, 4.5);
\draw (7, 0) rectangle node {$\mat{E}$} (10, 2);

\draw (1.5, 0) to [out=-90, in=180] (3, -1);
\node at (3, -0.75) {$n$};
\draw (3, -1) to [out=0, in=-90] (4.5, 0);

\draw (5.5, 0) to [out=-90, in=180] (7, -1);
\node at (7, -0.75) {$n$};
\draw (7, -1) to [out=0, in=-90] (8.5, 0);
\end{tikzpicture}
\end{aligned}
\end{equation*}

%% file: images/3-2-who-loves-whom.tex
\begin{equation*}
\begin{aligned}
\begin{tikzpicture} [scale=0.666, baseline={([yshift=-.5ex]current bounding box.center)},vertex/.style={anchor=south, circle,fill=black!25,minimum size=18pt,inner sep=2pt}]
\draw (0.8, 1.5)--(0.8, 0);
\node at (1.3, 0.5) {$| \E |$};
\draw (-0.8, 1.5)--(-0.8, 0);
\node at (-0.3, 0.5) {$| \E |$};

\node at (0, -1) {$\mat{q}$};
\draw (-2, 0)--(2, 0)--(0, -2)--(-2, 0);
\end{tikzpicture}
\quad = \quad
\begin{tikzpicture} [scale=0.666, baseline={([yshift=-.5ex]current bounding box.center)},vertex/.style={anchor=south, circle,fill=black!25,minimum size=18pt,inner sep=2pt}]

\node at (0.5, 3.5) {$Who$};

\node at (2, 2.5) {$| \E |$};
\draw (1.5, 2)--(1.5, 4.5);
\draw (0, 0) rectangle node {$\mat{E}$} (3, 2);

\node at (5.0, 3.5) {$loves$};
\draw (3.5, 3)--(6.5, 3)--(5.0, 4.5)--(3.5, 3);

\node at (5.75, 2.5) {$| \R |$};
\draw (5, 2)--(5, 3);
\draw (3.5, 0) rectangle node {$\mat{R}$} (6.5, 2);

\node at (10, 3.5) {$whom?$};

\node at (9.25, 2.5) {$| \E |$};
\draw (8.5, 2)--(8.5, 4.5);
\draw (7, 0) rectangle node {$\mat{E}$} (10, 2);

\draw (1.5, 0) to [out=-90, in=180] (3, -1);
\node at (3, -0.75) {$n$};
\draw (3, -1) to [out=0, in=-90] (4.5, 0);

\draw (5.5, 0) to [out=-90, in=180] (7, -1);
\node at (7, -0.75) {$n$};
\draw (7, -1) to [out=0, in=-90] (8.5, 0);
\end{tikzpicture}
\end{aligned}
\end{equation*}

%% file: 30-basic-discourse.tex
\begin{definition}
Take $a, b \in \set{0, 1}$, an \emph{atomic sentence} is a matrix in
$$ \enspace \bigg\{ \enspace \mat{G} \enspace \big(
\enspace \mat{sub} \otimes \mat{R} \ket{v} \otimes \mat{obj} \enspace
\big)
\enspace : \enspace | \E |^a \otimes | \E |^b \rightarrow 1
\enspace \bigg\}$$
where the verb $v$ ranges over the relations $R$.
If $a = 1$, the sentence is \emph{anaphoric} in the subject position and we take
$\mat{sub} = \mat{E} : | \E | \rightarrow n$.
If $a = 0$, $sub$ ranges over the encoded entities $\mat{E} \ket{s} : 1 \rightarrow n$ and the sentence subject is \emph{unambiguous}. Symmetrically for the object position:
\begin{equation*}
\mat{obj} \enspace = \enspace
\begin{cases}
\begin{aligned}
\enspace &\mat{E} \enspace &:& \enspace | \E | &\rightarrow& &n& \enspace
& \qquad \mathtt{if} \quad b &= 1
\\
\enspace &\mat{E} \ket{o} \enspace &:& \quad 1 &\rightarrow& &n&
& \qquad \mathtt{for} \enspace o &\in \E \enspace \mathtt{otherwise}
\end{aligned}
\end{cases}
\end{equation*}
We define a \emph{basic anaphoric discourse} as the matrix
$\mat{d} \enspace = \enspace \bigotimes_{i < l} \mat{t}_i \enspace : \enspace | \E |^k \rightarrow 1$
obtained by tensoring a list of $l$ atomic sentences
$$
\mat{t}_i \enspace : \enspace | \E |^{a_i} \otimes | \E |^{b_i} \rightarrow 1
$$
where $k \enspace = \enspace \sum_{i < l} \enspace (a_i + b_i)$,
\hspace{10pt} i.e. the discourse $\mat{d}$ contains $k$ anaphoric expressions in total.
\end{definition}
\begin{example}
The semantics of ``Spinoza influenced him. He discovered calculus." is given by the effect
\input{images/3-3-basic-discourse}
\end{example}

%% file: images/3-3-basic-discourse.tex
$$\begin{tikzpicture} [scale=0.666]
\node at (1.5, 3.5) {$Spin.$};
\draw (0.0, 3)--(3.0, 3)--(1.5, 4.5)--(0.0, 3);

\node at (2, 2.5) {$| \E |$};
\draw (1.5, 2)--(1.5, 3);
\draw (0, 0) rectangle node {$\mat{E}$} (3, 2);

\draw (1.5, 0) to [out=-90, in=180] (3, -1);
\node at (3, -0.75) {$n$};
\draw (3, -1) to [out=0, in=-90] (4.5, 0);

\node at (5.0, 3.5) {$infl.$};
\draw (3.5, 3)--(6.5, 3)--(5.0, 4.5)--(3.5, 3);

\node at (5.5, 2.5) {$| \R |$};
\draw (5, 2)--(5, 3);
\draw (3.5, 0) rectangle node {$\mat{R}$} (6.5, 2);

\draw (5.5, 0) to [out=-90, in=180] (7, -1);
\node at (7, -0.75) {$n$};
\draw (7, -1) to [out=0, in=-90] (8.5, 0);

\node at (9.5, 3.5) {$him.$};
\draw (8.5, 2)--(8.5, 4.5);

\node at (9, 2.5) {$| \E |$};
\draw (7, 0) rectangle node {$\mat{E}$} (10, 2);

\node at (12, 3.5) {$He$};
\draw (13, 2)--(13, 4.5);

\node at (13.5, 2.5) {$| \E |$};
\draw (11.5, 0) rectangle node {$\mat{E}$} (14.5, 2);

\draw (13, 0) to [out=-90, in=180] (14.5, -1);
\node at (14.5, -0.75) {$n$};
\draw (14.5, -1) to [out=0, in=-90] (16, 0);

\node at (16.5, 3.5) {$disc.$};
\draw (15.0, 3)--(18.0, 3)--(16.5, 4.5)--(15.0, 3);

\node at (17, 2.5) {$| \R |$};
\draw (16.5, 2)--(16.5, 3);
\draw (15, 0) rectangle node {$\mat{R}$} (18, 2);

\node at (20, 3.5) {$calc.$};
\draw (18.5, 3)--(21.5, 3)--(20.0, 4.5)--(18.5, 3);

\node at (20.5, 2.5) {$| \E |$};
\draw (20, 2)--(20, 3);
\draw (18.5, 0) rectangle node {$\mat{E}$} (21.5, 2);

\node at (22, 3.5) {$.$};

\draw (17, 0) to [out=-90, in=180] (18.5, -1);
\node at (18.5, -0.75) {$n$};
\draw (18.5, -1) to [out=0, in=-90] (20, 0);
\end{tikzpicture}$$

%% file: 4-understanding-as-process.tex
\section{Natural language understanding as a process}

We have seen that the semantics of \emph{basic anaphoric discourses} involving $k$ anaphoric expressions can be given in terms of knowledge graph queries of $k$ variables.
However, interpreting directly the effect $|\E|^2 \rightarrow 1$ corresponding to the discourse in the example above as a \textsc{Sparql} query of two variables will give counter-intuitive answers such as:
\begin{center}
``Spinoza influenced Wittgenstein. Newton discovered calculus."
\end{center}
whereas an English speaker will have naturally assumed that the two pronouns ``him" and ``he" have to refer to the same entity.
Without this constraint, maximising over the complex query would be equivalent to maximising the outcome for the two sentences independently.

\emph{Discourse representation theory} is a computational framework first introduced in \cite{Kamp93} for manipulating such natural language constraints.
In this paper, we model basic \emph{discourse representation structures} (DRS) in the sense of \cite{Abramsky14},
where entities are called ``discourse referents", atomic sentences ``literals" and pronouns ``variables".
Using the set of constraints formalised in terms of DRS, together with our encoding $\mat{T} = \mat{E} \otimes \mat{R} \otimes \mat{E}$ allows us to formalise \emph{probabilistic anaphora resolution} as an optimisation problem.

\begin{definition}
Given a \emph{matching function} $\mu : \set{0, \s \dots, \s k - 1} \rightarrow \E$ which assigns an entity in $\E$ to each pronoun in $\mat{d}$, the \emph{resolution of the ambiguity} in the discourse $\mat{d}$ is defined as the following scalar:
$$
\mathcal{A} (\mat{d}, \mu) \enspace = \enspace
d \enspace \bigotimes_{i < k} \ket{ \enspace \mu ( i ) \enspace } \enspace \in \S
$$
Given a set $\mathcal{D}(\mat{d}) \subseteq \E^{\{ 0 \dots k - 1 \}}$ of matching functions satisfying some DRS constraints for the anaphoric discourse $\mat{d}$, we define \emph{probabilistic anaphora resolution} as the following optimisation problem:
$$
\argmax_{\mu \in \mathcal{D}(\mat{d})} \enspace \mathcal{A}(\mat{d}, \mu)
$$
\end{definition}

The function $\mu$ can only give a static picture of the ambiguity resolution's result, instead we want a fine-grained diagram of the computational process involved.
We give a resource-aware reformulation of anaphora resolution in terms of an \emph{entity-store}, defined as the state:
$$
! \E \enspace = \enspace \bigotimes_{e \in \E} \ket{e}
 \enspace : \enspace 1 \rightarrow | \E |^{| \E |}
$$
From the matching function $\mu$ we construct a \emph{matching process} $! \mu$ which takes an entity-store $! \E$ and feeds the appropriate entities into the ambiguous sentences in $\mat{d}$.  
Graphically this construction amounts to drawing the graph of the function $\mu$, labeling all the wires with $| \E |$ and interpreting this as a diagram in $\mat{Mat}(\S)$.
When $i$ edges of the graph meet at a vertex in $\E$
--- i.e. when the same entity is referenced by $i$ distinct pronouns ---
we witness this with a Kronecker delta: the head of a $1$-input $i$-output spider.
Whereas in section 2 the multi-input one-output spider acted as intersection of meanings,
here the dual process copies an entity state and feeds it to its $i$ outputs i.e. for all $e \in E$ we have that:
\begin{align*}
(\delta_{| \E |}^{1, i}) \ket{e}
&\quad = \quad \quad \sum_{e' \in E} \ket{e'}^{(\otimes i)}
\bra{e'}\ket{e}
\\
&\quad = \quad \bigg( \sum_{e' \neq e} \ket{e'}^{(\otimes i)}
\underbrace{\bra{e'}\ket{e}}_{\text{= 0}} \bigg)
\enspace + \enspace
\ket{e}^{(\otimes i)}
\underbrace{\bra{e}\ket{e}}_{\text{= 1}}
\quad = \quad
\ket{e}^{(\otimes i)}
\end{align*}

\begin{theorem}
Given a matching function $\mu : \{ 0, \dots, k - 1 \} \rightarrow \E$ we can construct a matching process $! \mu : | \E |^{| \E |} \rightarrow | \E |^{k}$ such that:
$$
\forall \enspace \mat{d} \enspace : \enspace | \E |^{k} \rightarrow 1
\enspace \cdot \enspace
\mathcal{A} (\mat{d}, \mu) \enspace = \enspace \mat{d} (! \mu) (! \E )
$$
\end{theorem}
\begin{proof}
In set-theoretic terms, $ ! \E $ is the list of all the elements $e \in \E$, ordered by the indices of the corresponding vectors $\ket{e} : 1 \rightarrow | \E |$.
Then $! \mu$ is the incidence matrix of the graph for the function
$$
f_{\mu} \enspace (\enspace e_0, \enspace \dots, \enspace e_{| \E | - 1} \enspace)
 \enspace = \enspace
(\enspace e_{\mu (0)}, \enspace \dots, \enspace e_{\mu (k - 1)} \enspace)
\enspace : \enspace \E^{| \E |} \rightarrow \E^k
$$
In the case when the graph of the function $f_\mu$ is not planar, constructing $! \mu$ as a diagram in $\mat{Mat}(\S)$ requires a formalisation of the intuitive notion of ``wire-crossing", which can be obtained using the \emph{symmetry morphisms} of $\mat{Mat}(\S)$.
This falls outside the scope of our discussion but we refer the interested reader to any introduction to \emph{symmetric monoidal categories} (e.g. \cite{Coecke09b}) which are the suitable framework in this context.
\end{proof}
\begin{example}
Take $\E = \set{Descartes, Spinoza, Leibniz, Newton, calculus}$, let $\mu$ be the constant function $\set{0, 1 \mapsto Leibniz}$. We have:
\input{images/4-0-copying}
\end{example}

Even though our construction involves systems of exponential dimensions, the concrete computation it induces is linear in the size $n \times | \E |$ of the encoding matrix: we are only retrieving and copying vectors.
In effect, what this abstract formulation allows is to manipulate the knowledge graph, the entity encoding and matching explicitly, while hiding the bureaucratic process of labeling, indexing and wiring the entity-vectors required to compute complex queries.

We can now focus on what really is the core of the model: the interaction between the symbolic world of RDF and the statistical world of distributional semantics.
In diagrammatic terms, this is witnessed by the interaction of distinct spiders, i.e. Kronecker deltas over two different dimensions:
$\delta_{| \E |}$ for copying and matching entities,
$\delta_n$ for their intersection in the noun space $N$.


\begin{example}
We compute the resolution $\mathcal{A}(\mat{d}, \mu) = d (! \mu) (! \E) \in \mathbb{R}$ as the scalar:
\input{images/4-1-spinoza-influenced-him}
If we extract our encoding matrix $\mat{E}$ and knowledge graph $K$ from an encyclopedia, we should expect this scalar to be close to $1$.
Choosing $\mu = \{ 0, 1 \mapsto Descartes \}$ instead should make it close to $0$.
\end{example}

In order to get the class of all basic DRS, we need to assume that for every $r \in \R$ we have its negation $not ( r ) \in \R$ such that $E (x, not ( r ), y) \in K$ if and only if $(x, r, y) \notin K$.
Lifting this function $not : \R \rightarrow \R$ to the matrix of its graph $\ket{not} : | \R | \rightarrow | \R |$ and deriving a compositional semantics for natural language negations in $Mat_S$ from this is left for future work.
Nevertheless, if we restrict ourselves to the negation-free fragment, we can already give a semantics for basic anaphoric discourses in terms of the conjunctive fragment of \textsc{Sparql}.
Indeed, we can generate any query of the form 
$$
\mathtt{SELECT \enspace * \enspace WHERE} \enspace \{ \enspace
\mathtt{\langle \enspace  x, r, y \enspace \rangle
\enspace \in \enspace BGP \enspace}\}
$$
where $\mathtt{x}$ and $\mathtt{y}$ are taken to range over both entities and variables.
In the Semantic Web community, the set of triples $\mathtt{BGP}$ is also called a \emph{basic graph pattern} \cite{Stefanoni18}.

%% file: images/4-0-copying.tex
\begin{equation*}
\begin{aligned}
&
\begin{tikzpicture} [scale=0.666, baseline={([yshift=-.5ex]current bounding box.center)},vertex/.style={anchor=south, circle,fill=black!25,minimum size=18pt,inner sep=2pt}]
\draw [decorate,decoration={brace,amplitude=5},yshift=0.5]
(1.75, 4.75) -- (1.75, 7);
\node at (0, 6) {$(! \mu)$};

\draw[dashed] (2, 4.75) rectangle node {} (19.5, 7);

\node at (4.5, 6.375) {$\delta_{| \E |}^{1, 0}$};
\node [draw, circle, scale=0.5] (1) at (3.75, 6.0) {};
\draw (1)--(3.75, 7.5);

\node at (8, 6.375) {$\delta_{| \E |}^{1, 0}$};
\node [draw, circle, scale=0.5] (2) at (7.25, 6.0) {};
\draw (2)--(7.25, 7.5);

\node at (11.25, 5.375) {$\delta_{| \E |}^{1, 2}$};
\node [draw, circle, scale=0.5] (3) at (10.75, 6.0) {};
\draw (3) to [out=180, in=90] (8.5, 4.5);
\draw (3) to [out=0, in=90] (13, 4.5);
\draw (3)--(10.75, 7.5);

\node at (15, 6.375) {$\delta_{| \E |}^{1, 0}$};
\node [draw, circle, scale=0.5] (4) at (14.25, 6.0) {};
\draw (4)--(14.25, 7.5);

\node at (18.5, 6.375) {$\delta_{| \E |}^{1, 0}$};
\node [draw, circle, scale=0.5] (5) at (17.75, 6.0) {};
\draw (5)--(17.75, 7.5);

\draw [decorate,decoration={brace,amplitude=5},yshift=0.5]
(1.75, 7.25) -- (1.75, 9.25);
\node at (0, 8.25) {$(! E)$};

\draw[dashed] (2, 7.25) rectangle node {} (19.5, 9.25);

\node at (3.75, 8) {$Desc.$};
\draw (2.25, 7.5)--(5.25, 7.5)--(3.75, 9.0)--(2.25, 7.5);

\node at (7.25, 8) {$Spin.$};
\draw (5.75, 7.5)--(8.75, 7.5)--(7.25, 9.0)--(5.75, 7.5);

\node at (10.75, 8) {$Leib.$};
\draw (9.25, 7.5)--(12.25, 7.5)--(10.75, 9.0)--(9.25, 7.5);

\node at (14.25, 8) {$Newt.$};
\draw (12.75, 7.5)--(15.75, 7.5)--(14.25, 9.0)--(12.75, 7.5);

\node at (17.75, 8) {$calc.$};
\draw (16.25, 7.5)--(19.25, 7.5)--(17.75, 9.0)--(16.25, 7.5);
\end{tikzpicture}
\\
\\
&= \hspace{40pt}
\begin{tikzpicture} [scale=0.666, baseline={([yshift=-.5ex]current bounding box.center)},vertex/.style={anchor=south, circle,fill=black!25,minimum size=18pt,inner sep=2pt}]

\node at (4.5, 6.375) {$\delta_{| \E |}^{1, 0}$};
\node [draw, circle, scale=0.5] (1) at (3.75, 6.0) {};
\draw (1)--(3.75, 7.5);

\node at (8, 6.375) {$\delta_{| \E |}^{1, 0}$};
\node [draw, circle, scale=0.5] (2) at (7.25, 6.0) {};
\draw (2)--(7.25, 7.5);

\node at (11.25, 5.375) {$\delta_{| \E |}^{1, 2}$};
\node [draw, circle, scale=0.5] (3) at (10.75, 6.0) {};
\draw (3) to [out=180, in=90] (8.5, 4.5);
\draw (3) to [out=0, in=90] (13, 4.5);
\draw (3)--(10.75, 7.5);

\node at (15, 6.375) {$\delta_{| \E |}^{1, 0}$};
\node [draw, circle, scale=0.5] (4) at (14.25, 6.0) {};
\draw (4)--(14.25, 7.5);

\node at (18.5, 6.375) {$\delta_{| \E |}^{1, 0}$};
\node [draw, circle, scale=0.5] (5) at (17.75, 6.0) {};
\draw (5)--(17.75, 7.5);

\node at (3.75, 8) {$Desc.$};
\draw (2.25, 7.5)--(5.25, 7.5)--(3.75, 9.0)--(2.25, 7.5);
\draw[dashed] (2.125, 5.75) rectangle node {} (5.375, 9.125);

\node at (7.25, 8) {$Spin.$};
\draw (5.75, 7.5)--(8.75, 7.5)--(7.25, 9.0)--(5.75, 7.5);
\draw[dashed] (5.625, 5.75) rectangle node {} (8.875, 9.125);

\node at (10.75, 8) {$Leib.$};
\draw (9.25, 7.5)--(12.25, 7.5)--(10.75, 9.0)--(9.25, 7.5);

\node at (14.25, 8) {$Newt.$};
\draw (12.75, 7.5)--(15.75, 7.5)--(14.25, 9.0)--(12.75, 7.5);
\draw[dashed] (12.625, 5.75) rectangle node {} (15.875, 9.125);

\node at (17.75, 8) {$calc.$};
\draw (16.25, 7.5)--(19.25, 7.5)--(17.75, 9.0)--(16.25, 7.5);
\draw[dashed] (16.125, 5.75) rectangle node {} (19.375, 9.125);

\draw [decorate,decoration={brace,amplitude=5},yshift=0.5]
(20, 9.125)--(20, 5.75);

\node at (21, 7.5) {$= 1$};

\end{tikzpicture}
\\
\\
&= \hspace{132pt}
\begin{tikzpicture} [scale=0.666, baseline={([yshift=-.5ex]current bounding box.center)},vertex/.style={anchor=south, circle,fill=black!25,minimum size=18pt,inner sep=2pt}]
\node at (1.5, 3.5) {$Leib.$};
\draw (0.0, 3)--(3.0, 3)--(1.5, 4.5)--(0.0, 3);
\draw (1.5, 3.0)--(1.5, 2.0);
\node at (2, 2.5) {$|\E|$};
\end{tikzpicture}
\hspace{28pt}
\begin{tikzpicture} [scale=0.666, baseline={([yshift=-.5ex]current bounding box.center)},vertex/.style={anchor=south, circle,fill=black!25,minimum size=18pt,inner sep=2pt}]
\node at (1.5, 3.5) {$Leib.$};
\draw (0.0, 3)--(3.0, 3)--(1.5, 4.5)--(0.0, 3);
\draw (1.5, 3.0)--(1.5, 2.0);
\node at (2, 2.5) {$|\E|$};
\end{tikzpicture}
\end{aligned}
\end{equation*}

%% file: images/4-1-spinoza-influenced-him.tex
$$\begin{tikzpicture} [scale=0.666]
\node at (1.5, 3.5) {$Spin.$};
\draw (0.0, 3)--(3.0, 3)--(1.5, 4.5)--(0.0, 3);

\node at (2, 2.5) {$| \E |$};
\draw (1.5, 2)--(1.5, 3);
\draw (0, 0) rectangle node {$\mat{E}$} (3, 2);

\draw (1.5, 0) to [out=-90, in=180] (3, -1);
\node at (3, -0.75) {$n$};
\draw (3, -1) to [out=0, in=-90] (4.5, 0);

\node at (5.0, 3.5) {$infl.$};
\draw (3.5, 3)--(6.5, 3)--(5.0, 4.5)--(3.5, 3);

\node at (5.5, 2.5) {$| \R |$};
\draw (5, 2)--(5, 3);
\draw (3.5, 0) rectangle node {$\mat{R}$} (6.5, 2);

\draw (5.5, 0) to [out=-90, in=180] (7, -1);
\node at (7, -0.75) {$n$};
\draw (7, -1) to [out=0, in=-90] (8.5, 0);

\node at (9.5, 3.5) {$him.$};
\draw (8.5, 2)--(8.5, 4.5);

\node at (9, 2.5) {$| \E |$};
\draw (7, 0) rectangle node {$\mat{E}$} (10, 2);

\node at (12, 3.5) {$He$};
\draw (13, 2)--(13, 4.5);

\node at (13.5, 2.5) {$| \E |$};
\draw (11.5, 0) rectangle node {$\mat{E}$} (14.5, 2);

\draw (13, 0) to [out=-90, in=180] (14.5, -1);
\node at (14.5, -0.75) {$n$};
\draw (14.5, -1) to [out=0, in=-90] (16, 0);

\node at (16.5, 3.5) {$disc.$};
\draw (15.0, 3)--(18.0, 3)--(16.5, 4.5)--(15.0, 3);

\node at (17, 2.5) {$| \R |$};
\draw (16.5, 2)--(16.5, 3);
\draw (15, 0) rectangle node {$\mat{R}$} (18, 2);

\node at (20, 3.5) {$calc.$};
\draw (18.5, 3)--(21.5, 3)--(20.0, 4.5)--(18.5, 3);

\node at (20.5, 2.5) {$| \E |$};
\draw (20, 2)--(20, 3);
\draw (18.5, 0) rectangle node {$\mat{E}$} (21.5, 2);

\node at (22, 3.5) {$.$};

\draw (17, 0) to [out=-90, in=180] (18.5, -1);
\node at (18.5, -0.75) {$n$};
\draw (18.5, -1) to [out=0, in=-90] (20, 0);

\node at (4.5, 6.375) {$\delta_{| \E |}^{1, 0}$};
\node [draw, circle, scale=0.5] (1) at (3.75, 6.0) {};
\draw (1)--(3.75, 7.5);

\node at (8, 6.375) {$\delta_{| \E |}^{1, 0}$};
\node [draw, circle, scale=0.5] (2) at (7.25, 6.0) {};
\draw (2)--(7.25, 7.5);

\node at (11.25, 5.375) {$\delta_{| \E |}^{1, 2}$};
\node [draw, circle, scale=0.5] (3) at (10.75, 6.0) {};
\draw (3) to [out=180, in=90] (8.5, 4.5);
\draw (3) to [out=0, in=90] (13, 4.5);
\draw (3)--(10.75, 7.5);

\node at (15, 6.375) {$\delta_{| \E |}^{1, 0}$};
\node [draw, circle, scale=0.5] (4) at (14.25, 6.0) {};
\draw (4)--(14.25, 7.5);

\node at (18.5, 6.375) {$\delta_{| \E |}^{1, 0}$};
\node [draw, circle, scale=0.5] (5) at (17.75, 6.0) {};
\draw (5)--(17.75, 7.5);

\node at (3.75, 8) {$Desc.$};
\draw (2.25, 7.5)--(5.25, 7.5)--(3.75, 9.0)--(2.25, 7.5);

\node at (7.25, 8) {$Spin.$};
\draw (5.75, 7.5)--(8.75, 7.5)--(7.25, 9.0)--(5.75, 7.5);

\node at (10.75, 8) {$Leib.$};
\draw (9.25, 7.5)--(12.25, 7.5)--(10.75, 9.0)--(9.25, 7.5);

\node at (14.25, 8) {$Newt.$};
\draw (12.75, 7.5)--(15.75, 7.5)--(14.25, 9.0)--(12.75, 7.5);

\node at (17.75, 8) {$calc.$};
\draw (16.25, 7.5)--(19.25, 7.5)--(17.75, 9.0)--(16.25, 7.5);
\end{tikzpicture}$$

%% file: 9-conclusion.tex
\section*{Conclusion}

We have developed a process-theoretic framework for natural language understanding, making use of both statistical and logical resources to translate question answering and anaphora resolution as probabilistic knowledge base queries.
Using formal diagrams of matrices we give a high-level picture of NLP algorithms which are motivated both from the point of view of AI and of human cognition.

Even though we focused on $\S = \mathbb{B}$ and $\S = \mathbb{R}^+$, our discussion can be extended to any other semi-ring structure such as the unit interval $[0, 1]$ with $min$ and $max$ encoding fuzzy truth values.
It is also possible to incorporate further structure such as the convexity property used in conceptual space models of cognition, or taking proof relevance into account and record \emph{why} are entities related \cite{Coecke18}.

In the RDF language, variables can also serve as \emph{blank nodes} also called \emph{anonymous resources}: entities that are used in the computation but do not appear in the output of the query, e.g.~$\mathtt{y}$ in
$$
\mathtt{SELECT \enspace x \enspace WHERE \enspace \{ \enspace
	\langle \enspace x, loves, y \enspace \rangle, \enspace
	\langle \enspace y, loves, x \enspace \rangle \enspace \}}
$$
In $\mathtt{Mat}_S$, we conjecture that these would translate in terms of an \emph{entity-discarding} process $\delta_{| \E |}^{0, 1} : 1 \rightarrow | \E |$.

The translation from pregroup grammars to RDF triples is developed further in \cite{Antonin14},
in order to model count nouns (e.g.~``a cat") and relative pronouns (e.g.~``whose tail"), Delpeuch and Preller introduce a notion of categories with \emph{side effects} --- i.e. creation and update of blank nodes. It would be interesting to investigate the connections with our current work.

We would like to investigate how our model relates to concrete architectures for NLP, such as the Tensor Product Recurrent Networks (TPRNs) of \cite{Smolensky17}.
Another machine learning model which would fit our approach is that of \cite{Trouillon17}:
this method uses gradient descent over a landscape of complex-valued encoding matrices $\mat{E}$ and $\mat{R}$ to predict missing values in a knowledge graph $K$.
Linguistically, the passage to the alternative semi-ring of complex numbers
would allow to give two distinct encodings for entities: in our model the object-pronoun ``them" would be modelled as the complex conjugate of the subject ``they".
Modelling the active-to-passive switch with diagrams of complex matrices, as well as exploring potential implementations making use of quantum resources is left to later work.

In this paper we looked only at one discourse $d$ at a time, assuming a global matching process $! \mu$ and entity-store $! E$. 
Using methods from sheaf theory would allow to study the passage from local to global ambiguity and investigate the notion of \emph{contextuality} common to natural language, quantum mechanics and database theory \cite{Abramsky14b}.
This finer-grained analysis of the resolution process, as well as the implementation of other DRS operations, such as disjunction, implication and quantification, in the category $\mathtt{Mat}_\S$ is left for future work.